\documentclass{article}

	\PassOptionsToPackage{round}{natbib}

\usepackage[final]{neurips_2019}

\newif\ifsup
\supfalse\suptrue		

\usepackage[
backgroundcolor = White,textwidth=\marginparwidth]{todonotes}

\usepackage[markup=underlined]{changes}

\usepackage[utf8]{inputenc} 
\usepackage[T1]{fontenc}    
\usepackage{hyperref}       
\usepackage{url}            
\usepackage{booktabs}       
\usepackage{amsfonts}       
\usepackage{nicefrac}       
\usepackage{microtype}      

\usepackage{graphicx}
\usepackage{color}
\usepackage{rotating}
\usepackage{tabularx}
\usepackage{pdflscape}
\usepackage{amsmath,amsthm,amssymb}
\usepackage{algorithm,algpseudocode}
\usepackage{bbm,dsfont}
\usepackage{caption,subcaption}
\usepackage{cancel}
\usepackage{enumerate, cases}
\usepackage{thmtools,thm-restate}
\usepackage[none]{hyphenat}
\usepackage{natbib}
\usepackage{wrapfig}
\usepackage{mathtools}

\allowdisplaybreaks

\hypersetup{
	colorlinks	= true,		
	urlcolor     = blue,	 
	linkcolor	 = purple, 	 
	citecolor    = violet  	 
}
\usepackage[capitalize]{cleveref}

\DeclareMathOperator*{\argmin}{arg\,min}

\newcolumntype{C}[1]{>{\centering}m{#1}}
\newcommand{\EE}[1]{\mathbb{E}\left[#1\right]}

\newcommand{\Prob}[1]{\mathbb{P}\left\{#1\right\}}
\newcommand{\Regret}{\mathcal{R}}
\newcommand{\A}{\mathcal{A}_{Q}}
\newcommand{\R}{\mathbb{R}}
\newcommand{\N}{\mathbb{N}}
\newcommand{\one}[1]{\mathds{1}_{\left\{#1\right\}}}
\newcommand{\ceil}[1]{\left\lceil #1 \right\rceil}
\newcommand{\floor}[1]{\left\lfloor #1 \right\rfloor}

\newcommand{\PCSB}{\mathcal{P}_{\small{\text{CSB}}}}
\newcommand{\PMP}{\mathcal{P}_{\small{\text{MP}}}}
\newcommand{\PCoSB}{\mathcal{P}_{\small{\text{CoSB}}}}
\newcommand{\bmu}{\boldsymbol{\mu}}
\newcommand{\btheta}{\boldsymbol{\theta}}
\newcommand{\ba}{\boldsymbol{a}}

\newtheorem{thm}{Theorem}

\newtheorem{cor}{Corollary}

\newtheorem{defi}{Definition}

\title{	
	Censored Semi-Bandits: A Framework for Resource Allocation with Censored Feedback
}

\author{
	Arun Verma \\
	Department of IEOR\\
	IIT Bombay, India\\
	\texttt{v.arun@iitb.ac.in} \\
	\And
    Manjesh K. Hanawal \\
	Department of IEOR\\
	IIT Bombay, India\\
	\texttt{mhanwal@iitb.ac.in} \\
	\AND
	Arun Rajkumar \\
	Department of CSE\\
	IIT Madras, India\\
	\texttt{arunr@cse.iitm.ac.in} \\
	\And
	Raman Sankaran \\
	LinkedIn India\\
	Bengaluru, India\\
	\texttt{rsankara@linkedin.com} \\
}

\begin{document}
	
	\maketitle
	
	\begin{abstract}
        In this paper, we study {\em Censored Semi-Bandits}, a novel variant of the semi-bandits problem. The learner is assumed to have a fixed amount of resources, which it allocates to the arms at each time step. The loss observed from an arm is random and depends on the amount of resources allocated to it. More specifically, the loss equals zero if the allocation for the arm exceeds a constant (but unknown) threshold that can be dependent on the arm. Our goal is to learn a feasible allocation that minimizes the expected loss. The problem is challenging because the loss distribution and threshold value of each arm are unknown. We study this novel setting by establishing its `equivalence' to Multiple-Play Multi-Armed Bandits (MP-MAB) and Combinatorial Semi-Bandits. Exploiting these equivalences, we derive optimal algorithms for our setting using existing algorithms for MP-MAB and Combinatorial Semi-Bandits. Experiments on synthetically generated data validate performance guarantees of the proposed algorithms.
	\end{abstract}

	\section{Introduction}
	\label{sec:introduction}

Many real-life sequential resource allocation problems have a censored feedback structure. Consider, for instance, the problem of optimally allocating patrol officers (resources) across various locations in a city on a daily basis to combat \emph{opportunistic} crimes. Here, a perpetrator picks a location (e.g., a deserted street) and decides to commit a crime (e.g., mugging) but does not go ahead with it if a patrol officer happens to be around in the vicinity. Though the true potential crime rate depends on the latent decision of the perpetrator, one observes feedback only when the crime is committed. Thus crimes that were planned but not committed get censored.  This model of censoring is quite general and finds applications in several resources allocation problems such as police patrolling \citep{NSE10_curtin2010determining}, traffic regulations and enforcement \citep{AOR14_adler2014location,IJCAI17_rosenfeld2017security}, poaching control \citep{AAMAS16_nguyen2016capture,AAMAS18_gholami2018adversary}, supplier selection \citep{NIPS16_abernethy2016threshold}, advertisement budget allocation \citep{UAI14_lattimore2014optimal}, among many others.

Existing approaches that deal with censored feedback in resource allocation problems fall into two broad categories. \cite{NSE10_curtin2010determining,AOR14_adler2014location,IJCAI17_rosenfeld2017security} learn good resource allocations from historical data. \cite{AAMAS16_nguyen2016capture,AAMAS18_gholami2018adversary,AAMAS16_zhang2016using,IJCAI18_sinha2018stackelberg} pose the problem in a game-theoretic framework (opportunistic security games) and propose algorithms for optimal resource allocation strategies. While the first approach fails to capture the sequential nature of the problem, the second approach is agnostic to the user (perpetrator) behavior modeling. In this work, we balance these two approaches by proposing a simple yet novel threshold-based behavioral model, which we term as \emph{Censored Semi Bandits} (CSB). The model captures how a user opportunistically reacts to an allocation.

In the first variation of our proposed behavioral model, we assume the threshold (user behavioral) is uniform across arms (locations). We establish that this setup is `equivalent' to Multiple-Play Multi-Armed Bandits (MP-MAB), where a fixed number of armed is played in each round. We also study the more general variation where the threshold is arm dependent. We establish that this setup is equivalent to Combinatorial Semi-Bandits, where a subset of arms to be played is decided by solving a combinatorial $0$-$1$ knapsack problem.

Formally, we tackle the sequential nature of the resource allocation problem by establishing its equivalence to the MP-MAB and combinatorial semi-bandits framework. By exploiting this equivalence for our proposed threshold-based behavioral model, we develop novel resource allocation algorithms by adapting existing algorithms and provide optimal regret guarantees for the same.

\paragraph{Related Work:} The problem of resource allocation for tackling crimes has received significant interest in recent times. \cite{NSE10_curtin2010determining} employ a static maximum coverage strategy for spatial police allocation while  \cite{AAMAS16_nguyen2016capture} and \cite{AAMAS18_gholami2018adversary} study game-theoretic and adversarial perpetrator strategies. We, on the other hand, restrict ourselves to a non-adversarial setting. \citep{AOR14_adler2014location} and \cite{IJCAI17_rosenfeld2017security} look at traffic police resource deployment and consider the optimization aspects of the problem using real-time traffic, etc., which differs from the main focus of our work. \cite{AAMAS15_zhang2015keeping} investigate dynamic resource allocation in the context of police patrolling and poaching for opportunistic criminals.  Here they attempt to learn a model of criminals using a dynamic Bayesian network. Our approach proposes simpler and realistic modeling of perpetrators where the underlying structure can be exploited efficiently.

We pose our problem in the exploration-exploitation paradigm, which involves solving the MP-MAB and combinatorial 0-1 knapsack problem. It is different from the bandits with Knapsacks setting studied in \cite{JACM18_badanidiyuru2018bandits}, where resources get consumed in every round. The work of \cite{NIPS16_abernethy2016threshold} and \cite{ICML18_jain2018firing} are similar to us in the sense that they are also threshold-based settings. However, the thresholding we employ naturally fits our problem and significantly differs from theirs. Specifically, their thresholding is on a sample generated from an underlying distribution, whereas we work in a Bernoulli setting where the thresholding is based on the allocation. Resource allocation with semi-bandits feedback   \citep{UAI14_lattimore2014optimal,NIPS15_lattimore2015linear,ALT18_dagan18a} study a related but less general setup where the reward is based only on allocation and a hidden threshold. Our setting requires an additional unknown parameter for each arm, a `mean loss,' which also affects the reward.

Allocation problems in the combinatorial setting have been explored in \cite{JCSS12_cesa2012combinatorial,ICML13_chen2013combinatorial,NIPS14_rajkumar2014online,NIPS15_combes2015combinatorial,NIPS16_chen2016combinatorial,ICML18_wang2018thompson}. Even though these are not related to our setting directly, we derive explicit connections to a sub-problem of our algorithm to the setup of \cite{ICML15_komiyama2015optimal} and \cite{ICML18_wang2018thompson}.


	\section{Problem Setting}		
	\label{sec:problemSetting}

We consider a sequential learning problem where $K$ denotes the number of arms (locations), and $Q$ denotes the amount of divisible resources. The loss at arm $i \in [K]$ where $[K] := \{1, 2, \ldots, K\}$, is Bernoulli distributed with rate $\mu_i \in [0,1]$. Each arm can be assigned a fraction of resource that decides the feedback observed and the loss incurred from that arm -- if the allocated resource is smaller than a certain threshold\footnote{One could consider a smooth function instead of a step function, but the analysis is more involved, and our results need not generalize straightforwardly.}, the loss incurred is the realization of the arm, and it is observed. Otherwise, the realization is unobserved, and the loss incurred is zero. Let $\ba:=\{a_i: i\in [K]\}$, where $a_i \in [0,1]$, denotes the resource allocated to arm $i$. For each $i \in [K]$, let $\theta_i \in (0,1]$ denotes the threshold associated with arm $i$ and is such that a loss is incurred at arm $i$ only if $a_i< \theta_i$.
An allocation vector $\ba$ is said to be feasible if $\sum_{i \in [K]} a_i \leq Q$ and set of all feasible allocations is denoted as $\A$. The goal is to find a feasible resource allocation that results in a maximum reduction in the mean loss incurred.

In our setup, resources may be allocated to multiple arms. However, loss from each of the allocated arms may not be observed depending on the amount of resources allocated to them. We thus have a version of the partial monitoring system \citep{MOR06_cesa2006regret,ICML12_bartok2012partial,MOR14_bartok2014partial} with semi-bandit feedback, and we refer to it as censored semi-bandits (CSB). The vectors $\btheta:=\{\theta_i\}_{ i\in [K]}$ and $\bmu:=\{\mu_j\}_{i \in [K]}$ are unknown and identify an instance of CSB problem. Henceforth we identify a CSB instance as $P=(\bmu,\btheta,Q) \in [0,1]^{K}\times (0,1]^K \times \R_+$ and denote collection of CSB instances as $\mathcal{P}_{\small{\text{CSB}}}$. As $K = |\mu|$, $K$ is known (implicitly) from an instance of CSB. For simplicity of discussion, we assume that $\mu_1\le\mu_2 \le \ldots \le\mu_K$, but the algorithms are not aware of this ordering. For instance $P \in \PCSB$, the optimal allocation can be computed by the following $0$-$1$ knapsack problem
\begin{equation}
	\ba^\star  \in \argmin _{\ba \in \A} \sum_{i=1}^K \mu_i \one{a_i< \theta_i}.
\end{equation}

Interaction between the environment and a learner is given in Algorithm \ref{alg:protocol}.

\begin{algorithm}[H]
	For each round $t$: 
	\begin{enumerate}
		\item \textbf{Environment} generates a vector $\boldsymbol{X_t} = (X_{t,1}, X_{t,2},\ldots, X_{t,K}) \in \{0,1\}^K$, where $\EE{X_{t,i}}=\mu_i$ and the sequence $(X_{t,i})_{t\geq 1}$ is i.i.d. for all $i\in [K]$.
		\item \textbf{Learner} picks an allocation vector $\ba_t \in \A$.
		\item \textbf{Feedback and Loss:} The learner observes a random feedback $\boldsymbol{Y_t}=\{Y_{t,i}: i\in [K]\}$, where $Y_{t,i}=X_{t,i}\one{a_{t,i}<\theta_i}$ and incurs loss $\sum_{i \in [K]}Y_{t,i}$.
	\end{enumerate}
	\caption{CSB Learning Protocol on instance $(\bmu, \btheta, Q)$}
	\label{alg:protocol}
\end{algorithm}

The goal of the learner is to find a feasible resource allocation strategy at every round such that the cumulative loss is minimized. Specifically, 
we measure the performance of a policy that selects allocations $\{\ba_t\}_{t\geq1}$ over a period of $T$ in terms of expected (pseudo) regret given by
\begin{equation}
	\mathbb{E}[\Regret_T] = \sum_{t=1}^T\sum_{i=1}^K\mu_i \left(\one{a_{t,i}< \theta_i} -  \one{a^\star_i< \theta_i}\right).
\end{equation}

A good policy should have sub-linear expected regret, i.e., $\EE{\Regret_T}/T \rightarrow 0$ as $T \rightarrow \infty$.

	\section{Identical Threshold for All Arms}
	\label{sec:same_theta}

In this section, we focus on the particular case of the censored semi bandits problem where $\theta_i=\theta_c$ for all $i \in [K]$. With abuse of notation, we continue to denote an instance of CSB with the same threshold as $(\bmu, \theta_c, Q)$, where $\theta_c \in (0,1]$ is the same threshold. 

\begin{defi}
	For a given loss vector $\bmu$ and resource $Q$, we say that thresholds $\theta_c$ and $\hat{\theta}_c$ are {\em allocation equivalent} if the following holds:
	\begin{equation*}
		\min_{\ba \in \A} \sum_{i=1}^K \mu_i \one{a_i< \theta_c}  = 
		\min_{\ba \in \A} \sum_{i=1}^K \mu_i \one{a_i< \hat{\theta}_c}.
	\end{equation*}
\end{defi}

Though $\theta_c$ can take any value in the interval $(0,1]$, an allocation equivalent to it can be confined to a finite set. The following lemma shows that a search for an allocation equivalent can be restricted to $\ceil{K-Q+1}$ elements.
\begin{restatable}{lem}{ThetaSet}
	\label{lem:thetaSet}
	For any $\theta_c \in (0,1]$ and $Q \ge \theta_c$, let $M=\min\{\floor{Q/\theta_c}, K\}$ and $\hat{\theta}_c=Q/M$. 
	Then  $\theta_c$ and $\hat{\theta}_c$ are allocation equivalent. Further, $\hat{\theta}_c \in \Theta$ where $\Theta = \{Q/K, Q/(K-1), \cdots, \min\{1, Q\}\}$. 
\end{restatable}

Let $M= \min\{\lfloor Q/\theta_c\rfloor, K\}$. In the following, when arms are sorted in the increasing order of mean losses, we refer to the last $M$ arms as the \emph{bottom-}$M$ arms and the remaining arms as \emph{top-}$(K-M)$ arms. It is easy to note that an optimal allocation with the same threshold $\theta_c$ is to allocate $\theta_c$ resource to each of the bottom-$M$ arms and allocate the remaining resources to the other arms. \cref{lem:thetaSet} shows that range of allocation equivalent $\hat{\theta}_c$ for any instance $(\bmu,\theta_c,Q)$ is finite. Once this value is found, the problem reduces to identifying the bottom-$M$ arms and assigning resource $\hat{\theta}_c$ to each one of them to minimize the mean loss. The latter part is equivalent to solving a Multiple-Play Multi-Armed Bandits problem, as discussed next.

\subsection{Equivalence to Multiple-play Multi-armed Bandits}
In stochastic Multiple-Play Muti-Armed Bandits (MP-MAB), a learner can play a subset of arms in each round known as superarm \citep{TAC1987_MultiPlayBandits_Anatharam}. The size of each superarm is fixed (and known). The mean loss of a superarm is the sum of the means of its constituent arms. In each round, the learner plays a superarm and observes the loss from each arm played (semi-bandit). The goal of the learner is to select a superarm that has the smallest mean loss. A policy in MP-MAB selects a superarm in each round based on the past information. Its performance is measured in terms of regret defined as the difference between cumulative loss incurred by the policy and that incurred by playing an optimal superarm in each round. Let $(\bmu, m) \in [0,1]^K \times \N_+$ denote an instance of MP-MAB where $\bmu$ denote the mean loss vector, and $m \le K$ denotes the size of each superarm. Let $\PCSB^c \subset \PCSB$ denote the set of CSB instances with the same threshold for all arms. For any  $(\bmu,\theta_c, Q) \in \PCSB^c$ with $K$ arms and known threshold $\theta_c$, let $(\bmu, m)$ be an instance of MP-MAB with $K$ arms and each arm has the same Bernoulli distribution as the corresponding arm in the CSB instance with $m=K-M$, where $M=\min\{\floor{Q/\theta_c}, K\}$ as earlier. Let $\PMP$ denote the set of resulting MP-MAB problems and $f: \PCSB \rightarrow \PMP$ denote the above transformation.

Let $\pi$ be a policy on $\PMP$. $\pi$ can also be applied on any $(\bmu,\theta_c, Q) \in \PCSB^c$ with known $\theta_c$ to decide which set of arms are allocated resource as follows: in round $t$, let the information $(L_1, Y_1, L_2,Y_2, \ldots, L_{t-1}, Y_{t-1})$ collected on an CSB instance, where $L_s$ is the set of $K-M$ arms where no resource is applied and $Y_s$ is the samples observed from these arms. In round $t$, this information is given to $\pi$ which returns a set $L_t$ with $K-M$ elements. Then all arms other than arms in $L_t$ are given resource $\theta_c$. Let this policy on $(\bmu,\theta_c, Q) \in \PCSB^c$ be denoted as $\pi^\prime$. In a similar way a policy $\beta$ on $\PCSB$ can be adapted to yield a policy for $\PMP$ as follows: in round $t$, let the information $(S_1, Y_1, S_2,Y_2, \ldots, S_{t-1}, Y_{t-1})$ collected on an MP-MAB instance, where $S_s$ is the superarm played in round $s$ and $Y_s$ is the associated loss observed from each arms in $S_s$,  is given to $\pi$ which returns a set $S_t$ of $K-M$ arms where no resources has to be applied. The superarm corresponding to $S_t$ is then played. Let this policy on $\PMP$ be denoted as $\beta^\prime$. Note that when $\theta_c$ is known, the mapping is invertible. The next proposition gives regret equivalence between the MP-MAB problem and CSB problem with a known same threshold.

\begin{restatable}{prop}{RegretEquiST}
	\label{prop:RegretEquiST}
	Let $P:=(\bmu,\theta_c, Q) \in \PCSB^c$ with known $\theta_c$ then the regret of $\pi$ on $P$ is same as the regret of $\pi^\prime$ on $f(P)$. Similarly, let $P^\prime:=(\bmu,m) \in \PMP$, then regret of a policy $\beta$ on $P^\prime$ is same as the regret of $\beta^\prime$ on $f^{-1}(P^\prime)$. Thus the set $\PCSB$ with a known $\theta_c$ is 'regret equivalent' to $\PMP$, i.e., $\Regret(\PCSB^c)=\Regret(\PMP)$. 
\end{restatable}

{\bf Lower bound}: As a consequence of the above equivalence and one to one correspondence between the MP-MAB and CSB with the same threshold (known), a lower bound on MP-MAB is also a lower bound on CSB with the same threshold. Hence the following lower bound given for any strongly consistent algorithm \cite[Theorem 3.1]{TAC1987_MultiPlayBandits_Anatharam} is also a lower bound on the CSB problem with the same threshold:
\begin{equation}
	\label{eqn:LowerBound}
	\lim_{T\rightarrow \infty}\Prob{\frac{\mathbb{E}[\Regret_T]}{\log T} \geq \mbox{$\sum_{i \in [K]\setminus [K-M]}$}\frac{(1-o(1))(\mu_i-\mu_{K-M})}{d(\mu_{K-M}, \mu_i)}} = 1,
\end{equation}
where $d(p,q)$ is the Kullback-Leibler (KL) divergence between two Bernoulli distributions with parameter $p$ and $q$. Also note that we are in loss setting.

The above proposition suggests that any algorithm which works well for the MP-MAB also works well for the CSB once the threshold is known. Hence one can use algorithms like MP-TS \citep{ICML15_komiyama2015optimal} and ESCB \citep{NIPS15_combes2015combinatorial} once an allocation equivalent to $\theta_c$ is found. MP-TS uses Thompson Sampling, whereas ESCB uses UCB (Upper Confidence Bound) and KL-UCB type indices. One can use any of these algorithms. But we adapt MP-TS to our setting as it gives the better empirical performance and shown to achieve optimal regret bound for Bernoulli distributions.

\subsection{Algorithm: \ref{alg:CSB-ST}}
We develop an algorithm named \ref{alg:CSB-ST} for solving the Censored Semi Bandits problem with Same Threshold. It exploits result in \cref{lem:thetaSet} and equivalence established in Proposition \ref{prop:RegretEquiST} to learn a good estimate of threshold and minimize the regret using a MP-MAB algorithm. \ref{alg:CSB-ST} consists of two phases, namely, threshold estimation and regret minimization.
\begin{algorithm}[H] 
	\renewcommand{\thealgorithm}{\bf CSB-ST}
	\floatname{algorithm}{}
	\caption{Algorithm for solving the Censored Semi Bandits problem with Same Threshold}
	\label{alg:CSB-ST}
	\begin{algorithmic}[1]
		\Statex 
		\textbf{Input:} $K, Q, \delta, \epsilon$
		\vspace{1.5mm}
		\Statex \textbackslash\textbackslash \textbf{ Threshold Estimation Phase} \textbackslash\textbackslash
		\vspace{0.5mm}
		\State Initialize $C=0, l=0, u = |\Theta|, i = \ceil{u/2}$
		\State Set $\Theta$ as \cref{lem:thetaSet}, $T_{\theta_s}= 0, W = {\log(\log_2(|\Theta|)/\delta)}/{(\max\{1, \floor{Q}\}\log(1/(1-\epsilon))})$ 
		\While{$i \ne u$}
			\State Set $\hat{\theta}_c = \Theta[i]$
			\State $A_t \leftarrow $ first ${Q}/{\hat{\theta}_c}$ arms. Allocate $\hat{\theta}_c$ resource to each arm $i \in A_t$
			\State If loss observed for any arm $i \in A_t$ then set $l=i, i = l + \ceil{\frac{u - l}{2}}, C=0$ else $C = C + 1$
			\State If $C = W$ then set $u=i,i = u - \floor{\frac{u - l}{2}}, C=0$
			\State $T_{\theta_s} = T_{\theta_s}+1$
		\EndWhile
		\vspace{1.5mm}
		\Statex \textbackslash\textbackslash \textbf{ Regret Minimization Phase} \textbackslash\textbackslash
		\vspace{0.5mm}
		\State Set $M = Q/\hat\theta_c$ and $\forall i \in [K]: S_i = 1, F_i = 1$ 
		\For{$t=T_{\theta_s}+1, T_{\theta_s} + 2, \ldots, T$}
			\State $\forall i \in [K]: \hat{\mu}_i(t) = \beta(S_i, F_i)$
			\State $L_t \leftarrow$ (K-M) arms with smallest  estimates
			\State $\forall i \in L_t:$ allocate no resource to arm $i$. $\forall j \in K\setminus L_t:$ allocate $\hat\theta_c$ resources to arm $j$
			\State $\forall i \in L_t:$ observe $X_{t,i}$. Update $S_i = S_i+X_{t,i}$ and $F_i = F_i+1-X_{t,i}$
		\EndFor
	\end{algorithmic}
\end{algorithm}

\textbf{Threshold Estimation Phase:}
This phase finds a threshold $\hat{\theta}_c$ that is allocation equivalent to the underlying threshold $\theta_c$ with high probability by doing a binary search over the set $\Theta=\{Q/K, Q/(K-1), \dots, \min\{1, Q\}\}$.  The elements of $\Theta$ are arranged in increasing order and are candidates for $\theta_c$. The search starts by taking $\hat{\theta}_c$ to be the middle element in $\Theta$ and allocating $\hat{\theta}_c$ resource to first $Q/\hat\theta_c$ arms (denoted as set $A_t$ in Line $5$). If a loss is observed at any of these arms, it implies that $\hat{\theta}_c$ is an underestimate of $\hat{\theta}_c$. All the candidates lower than the current value of $\hat{\theta}_c$ in $\Theta$ are eliminated, and the search is repeated in the remaining half of the elements again by starting with the middle element (Line $6$). If no loss is observed for consecutive $W$ rounds, then $\hat{\theta}_i$ is possibly an overestimate. Accordingly, all the candidates larger than the current value of $\hat{\theta}_c$ in $\Theta$ are eliminated, and the search is repeated starting with the middle element in the remaining half (Line $7$). The variable $C$ keeps track of the number of the consecutive rounds for which no loss is observed. It changes to $0$ either after observing a loss or if no loss is observed for consecutive $W$ rounds.

Note that if $\hat{\theta}_c$ is an underestimate and no loss is observed for consecutive $W$ rounds, then $\hat{\theta}_c$ will be reduced, which leads to a wrong estimate of $\hat{\theta}_c$. To avoid this, we set the value of $W$ such that the probability of happening of such an event is low. The next lemma gives a bound on the number of rounds needed to find allocation equivalent for threshold $\theta_c$ with high probability.

\begin{restatable}{lem}{sameThresholdEstRounds}
	\label{lem:sameThresholdEstRounds}
	Let $(\bmu, \theta_c, Q)$ be an CSB instance such that $\mu_1 \geq  \epsilon>0$. Then with probability at least $1-\delta$, the number of rounds needed by threshold estimation phase of \ref{alg:CSB-ST} to find the  allocation equivalent for threshold $\theta_c$ is bounded as 
	\begin{equation*}
		T_{\theta_s}\le \frac{\log(\log_2(|\Theta|)/\delta)}{\max\{1, \floor{Q}\}\log\left({1}/{(1-\epsilon)}\right)}\log_2(|\Theta|).
	\end{equation*}
\end{restatable}

Once $\hat\theta_c$ is known, $\bmu$ needs to be estimated. The resources can be allocated such that no losses observe for maximum $M$ arms. As our goal is to minimize the mean loss, we have to select $M$ arms with the highest mean loss and then allocate $\hat\theta_c$ to each of them. It is equivalent to find $K-M$ arms with the least mean loss then allocate no resources to these arms and observe their losses. These losses are then used for updating the empirical estimate of the mean loss of arms.

\textbf{Regret Minimization Phase:} 
The regret minimization phase of \ref{alg:CSB-ST} adapts Multiple-Play Thompson Sampling (MP-TS) \citep{ICML15_komiyama2015optimal} for our setting. It works as follows: initially we set the prior distribution of each arms as the Beta distribution $\beta(1, 1)$, which is same as Uniform distribution on $[0,1]$. $S_i$ represents the number of round when loss is observed whereas $F_i$ represents the number of round when loss is not observed. Let $S_i(t)$ and $F_i(t)$ denotes the values of $S_i$ and $F_i$ in the beginning of round $t$. In round $t$, a sample $\hat\mu_i$ is independently drawn from $\beta(S_i(t), F_i(t))$ for each arm $i \in [K]$.  $\hat\mu_i$ values are ranked by their increasing values. The first $K-M$ arm assigned no resources (denoted as set $L_t$ in Line $13$) while each of the remaining $M$ arms are allocated $\hat\theta_c$ resources. The loss $X_{t,i}$ is observed for each arm $i \in L_t$ and then success and failure counts are updated by setting $S_i = S_i + X_{t,i}$ and $F_i = F_i + 1 - X_{t,i}$.

\subsubsection{Regret Upper Bound}
\label{sssec:sameThetaRegretBounds}
For instance $(\bmu,\theta_c, Q)$ and any feasible allocation $\ba \in \A$, we define $\nabla_{\ba} = \sum_{i=1}^K\mu_i\big(\one{a_i < \theta_i} - \one{a_i^\star < \theta_i}\big)$ and $\nabla_m = \max_{\ba \in \A } \nabla_{\ba}$. We are now ready the state the regret bounds.

\begin{restatable}{thm}{regretSameThreshold}
	\label{thm:regretSameThreshold}
	Let $\mu_1\geq \epsilon>0$, $W_T = {\log(T\log_2(|\Theta|))}/{\max\{1, \floor{Q}\}\log(1/(1-\epsilon))}$, $\mu_{K-M} < \mu_{K-M+1},$ and $T>W_T\log_2(|\Theta|)$. Set $\delta=1/T$ in \ref{alg:CSB-ST}. Then the regret of \ref{alg:CSB-ST} is upper bound as
	\begin{equation*}
		\EE{\Regret_T} \le W_T\log_2{(|\Theta|)}\nabla_m   + O\left((\log T)^{{2}/{3}}\right) + \sum_{i \in [K]\setminus [K-M]} \frac{(\mu_i-\mu_{K-M} )\log {T}}{d( \mu_{K-M},\mu_i)}.
	\end{equation*}
\end{restatable}
The first term in the regret bound in Theorem \ref{thm:regretSameThreshold} corresponds to the length of the threshold estimation phase, and the remaining parts correspond to the expected regret in the regret minimization phase.

Note that the assumption $\mu_1\ge\epsilon$ is only required to guarantee that the threshold estimation phase terminates in the finite number of rounds. This assumption is not needed to get the bound on expected regret in the regret minimization phase. The assumption $\mu_{K-M} < \mu_{K-M+1}$ ensures that Kullback-Leibler divergence in the bound is well defined. This assumption is also equivalent to assume that the set of top-$M$ arms is unique.

\begin{cor}
	\label{cor:OptimalBoundST}
	The regret of \ref{alg:CSB-ST} is asymptotically optimal.
\end{cor}

The proof of Corollary \ref{cor:OptimalBoundST} follows by comparing the above bound with the lower bound in \cref{eqn:LowerBound}.

	\section{Different Thresholds}
	\label{ssec:different_theta}

In this section, we consider a more difficult problem where the threshold may not be the same for all arms. Let $KP(\bmu,\btheta, Q)$ denote a 0-1 knapsack problem with capacity $Q$ and $K$ items where item $i$ has weight $\theta_i$ and value $\mu_i$. Our first result gives the optimal allocation for an instance in $\PCSB$:
\begin{restatable}{prop}{diffThetaOptiSoln}
	\label{prop:diffThetaOptiSoln}
	Let $P=(\bmu,\btheta,Q) \in \PCSB$. Then the optimal allocation for $P$ is a solution of $KP(\bmu,\btheta, Q)$ problem.
\end{restatable}
The proof of the above proposition and computational issues of the 0-1 knapsack with fractional values of it are given in the supplementary. We next discuss when two threshold vectors are allocation equivalent. Extending the definition of allocation equivalence to threshold vectors, we say that two vectors $\btheta$ and $\hat{\btheta}$ are allocation equivalent if minimum mean loss in instances $(\bmu,\btheta, Q)$ and $(\bmu, \hat{\btheta}, Q)$ are the same for any loss vector $\bmu$ and resource $Q$. This equivalence allows us to look for estimated thresholds within some tolerance. We need the following notations to make this formal.

For an instance $P:=(\bmu,\btheta, Q)$, recall that $\ba^\star=(a_1^\star, \ldots, a_K^\star)$ denotes the optimal allocation. Let $r = Q - \sum_{i: a_i^\star \ge \theta_i}\theta_i$, where $r$ is the residual resources after the optimal allocation. Define $\gamma:=r/K$. Any problem instance with $\gamma= 0$ becomes a `hopeless' problem instance as the only vector that is allocation equivalent to $\btheta$ is $\btheta$ itself, which demands $\theta_i$ values to be estimated with full precision to achieve optimal allocation. However, if $\gamma>0$, an optimal allocation can be still be found with small errors in the estimates of $\theta_i$ as shown next.
\begin{restatable}{lem}{diffTheteEst}
	\label{lem:diffTheteEst}
	Let $ \gamma>0$ and $\forall i \in [K]: \hat\theta_i \in [\theta_i, \theta_i+\gamma]$. Then for any $\bmu \in [0,1]^K$ and $Q$, the  $\btheta$ and $\hat{\btheta}$ are allocation equivalent.
\end{restatable}

The proof follows by an application of Theorem 3.2 in \cite{DO13_hifi2013sensitivity} which gives conditions for two weight vectors $\btheta_1$ and $\btheta_2$ to have the same solution in $KP(\bmu,\btheta_1,Q)$ and $KP(\bmu,\btheta_2, Q)$ for any $\bmu$ and $Q$. Once we estimate the threshold $\btheta$ with accuracy so that the estimate $\hat{\btheta}$ is an allocation equivalent of $\btheta$, the problem is equivalent to solving the $KP(\bmu,\hat{\btheta}, Q)$. The latter part is equivalent to solving a combinatorial semi-bandits as we establish next. Combinatorial semi-bandits are the generalization of MP-MAB, where the size of the superarms need not be the same in each round.

\begin{restatable}{prop}{MultiThetaEquivalence}
	\label{prop:MultiThetaEquivalence}
	The CSB problem with known threshold vector $\btheta$ is regret equivalent to a combinatorial semi-bandits where Oracle uses $KP(\bmu,\btheta, Q)$ to identify the optimal subset of arms.
\end{restatable}

\subsection{Algorithm: \ref{alg:CSB-DT}}
We develop an algorithm named \ref{alg:CSB-DT} for solving the Censored Semi Bandits problem with Different Threshold. It exploits the result of \cref{lem:diffTheteEst} and equivalence established in \cref{prop:MultiThetaEquivalence} to learn a good estimate of the threshold for each arm and minimizes the regret using an algorithm from combinatorial semi-bandits. \ref{alg:CSB-DT} also consists of two phases: threshold estimation and regret minimization.

\begin{algorithm}[H] 
	\renewcommand{\thealgorithm}{\bf CSB-DT}
	\floatname{algorithm}{}
	\caption{Algorithm for solving the Censored Semi Bandits problem with Different Threshold}
	\label{alg:CSB-DT}
	\begin{algorithmic}[1]
		\Statex\textbf{Input:} $K, Q, \delta, \epsilon, \gamma$
		\vspace{1mm}
        \Statex \textbackslash\textbackslash \textbf{ Threshold Estimation Phase} \textbackslash\textbackslash
        \vspace{0.5mm}
		\State Initialize $\forall i \in [K]: \theta_{l,i}= 0, \theta_{u,i}= 1, \theta_{g,i}= 0, C_i =0$. 
		\State Set $T_{\theta_d}=0, W = {\log(K\log_2(\lceil 1 + 1/\gamma\rceil)/\delta)}/ {\log(1/(1-\epsilon))}$
		\State $\forall i \in [\floor{Q}]:$ allocate $\hat\theta_i = 0.5$ resource. $\forall j \in [\lfloor Q \rfloor +1,K]:$ allocate $\hat\theta_j =\frac{Q-\floor{Q}/2}{K-\floor{Q}}$ resources
		\While {$\theta_{g,i}= 0$ for any $i \in [K]$}
			\For{$i = 1, \ldots, K$} 
				\If {loss observe for arm $i$ with $\theta_{g,i}=0$ and $\theta_{l,i} < \hat\theta_i$}
					\State Set $\theta_{l,i}= \hat\theta_i, \hat\theta_i = (\theta_{u,i}+ \theta_{l,i})/2, C_i=0$. If available allocate resource $\hat\theta_i$
				\Else
					\State If allocated resources is $\hat\theta_i$ then reset $C_i = C_i + 1$
					\If {$C_i = W$ and $\theta_{g,i}= 0$}
						\State Set $\theta_{u,i}= \hat\theta_i, \hat\theta_i = (\theta_{u,i} + \theta_{l,i})/2$, $C_i=0$. If available allocate resource $\hat{\theta}_i$
						\State If $\theta_{u,i} - \theta_{l,i}\le \gamma$ then set $\theta_{g,i}=1$ and $\hat\theta_i=\theta_{u,i}$
					\EndIf
				\EndIf
			\EndFor
			\While{free resources are available}
				\State Allocate $\hat\theta_i$ resources to a new randomly chosen arm $i$ from the arms having $\theta_{g,i}=1$
			\EndWhile
			\State $T_{\theta_d} = T_{\theta_d} + 1$
		\EndWhile
        \vspace{1mm}
		\Statex \textbackslash\textbackslash \textbf{ Regret Minimization Phase} \textbackslash\textbackslash
		\vspace{0.5mm}
		\State $\forall i \in [K]: S_i = 1, F_i = 1$
		\For{$t=T_{\theta_d}+1, T_{\theta_d} + 2, \ldots, T$}
  				\State $\forall i \in [K]: \hat{\mu}_i(t) \leftarrow \text{Beta}(S_i, F_i)$
  				\State $L_t \leftarrow$ Oracle$\big( KP(\hat\bmu(t), \hat\btheta,Q)\big)$ 
  				\State $\forall i \in L_t:$ allocate no resource to arm $i$. $\forall j \in K\setminus L_t:$ allocate $\hat\theta_j$ resources to arm $j$
  				\State $\forall i \in L_t:$ observe $X_{t,i}$. Update $S_i = S_i+X_{t,i}$ and $F_i = F_i+1-X_{t,i}$
		\EndFor
	\end{algorithmic}
\end{algorithm}

\textbf{Threshold Estimation Phase:} This phase finds a threshold that is allocation equivalent of $\btheta$ with high probability. This is achieved by ensuring that $\hat\theta_i \in [\theta_i, \theta_i+\gamma]$ for all $i$ (\cref{lem:diffTheteEst}). For each arm $i\in  [K]$ a binary search is performed over the interval $[0, 1]$ by maintaining variables $\hat{\theta}_i,\theta_{l,i}, \theta_{u,i},\theta_{g,i}$, and $C_i$ where $\hat{\theta}_i$ is the current estimate of $\theta_i; \theta_{l,i}$ and $\theta_{u,i}$ denote the lower and upper bound of the binary search region for arm $i$; and $\theta_{g,i}$ indicates whether current estimate lies in the interval  $[\theta_i, \theta_i+\gamma]$. In each round, the threshold estimate of arms are first updated sequentially and then tested on their respective arms. $C_i$ keeps counts of consecutive rounds without no loss for $\hat\theta_i$. It changes to $0$ either after observing a loss or if no loss is observed for consecutive $W$ rounds.

The threshold estimation phase starts with allocating $0.5$ resource for first $\floor{Q}$ arms and $(Q-\floor{Q}/2)/(K-\floor{Q})$ for the remaining arms (Line $3$). In each round, allocated resource are applied on each arm and based on the observations their estimates and the allocated resource are updated sequentially starting from $1$ to $K$ as follows. If a loss is observed for arm $i$ having bad threshold estimate ($\theta_{g,i}=0$) and $\theta_{l,i} < \hat\theta_i$, then it implies that $\hat\theta_i$ is an underestimate of $\theta_i$ and the following actions are performed -- 1) lower end of search region is increased to $\hat\theta_i$, i.e., $\theta_{l,i}=\hat\theta_i$; 2) its estimate $\hat\theta_i$ is set to $(\theta_{u,i} + \theta_{l,i})/2$; 3) if available allocate $\hat\theta_i$ resource to arm $i$; and 4) set $C_i=0$ (Line $7$).

If no loss is observed after allocating $\hat\theta_i$ resources for $W$ successive rounds for arm $i$ with bad threshold estimate, then it implies that $\hat\theta_i$ is overestimated and following actions are performed -- 1) the upper end of the search region is changed to $\hat\theta_i$, i.e, $\theta_{u,i}=\hat{\theta}_i$; 2) its estimate $\hat\theta_i$ is set to $(\theta_{u,i} + \theta_{l,i})/2$; and 3) if available allocate $\hat\theta_i$ resource to arm $i$ (Line $11$). Further, whether goodness of $\hat{\theta}_i$ holds, i.e., $\hat{\theta}_i \in [\theta_i, \theta_i+\gamma]$ is checked by condition $\theta_{u,i}  - \theta_{l,i}\le \gamma$. If the condition holds, the threshold estimation of arm is within desired accuracy and this is indicated by setting $\theta_{g,i}$ to 1 and $\hat\theta_i=\theta_{u,i}$ (Line $12$). Any unassigned resources are given to randomly chosen arms having good threshold estimates (all arms with $\theta_{g,i}=1$) where each arm $i$ gets only $\hat\theta_i$ resources (Line $17$).

The value of $W$ in \ref{alg:CSB-DT} is set such that the probability of estimated threshold does not lie in $[\theta_i, \theta_i+\gamma]$ for all arms is bounded by $\delta$. Following lemma gives the bounds on the number of rounds needed to find the allocation equivalent for threshold vector $\btheta$ with high probability.

\begin{restatable}{lem}{MultiTheta}
	\label{lem:MultiTheta}
	Let $(\bmu,\btheta, Q)$ be an instance of CSB such that $\gamma>0$ and $\mu_1 \geq  \epsilon>0$. Then with probability at least $1-\delta$, the number of rounds needed by threshold estimation phase of \ref{alg:CSB-DT} to find the allocation equivalent for threshold vector $\btheta$ is bounded as 
	\begin{equation*}
		T_{\theta_d} \le \frac{K\log( K \log_2(\ceil{1 +{1}/{\gamma}})/\delta)} {\max\{1, \floor{Q}\}\log(1/(1-\epsilon))} {\log_2 (\ceil{1 +  {1}/{\gamma}})}.
	\end{equation*}
\end{restatable}

\textbf{Regret Minimization Phase:} 
For this phase, we could use an algorithm that works well for the combinatorial semi-bandits, like SDCB \citep{NIPS16_chen2016combinatorial} and CTS \citep{ICML18_wang2018thompson}. CTS uses Thompson Sampling, whereas SDCB uses the UCB type index. We adopt the CTS to our setting due to better empirical performance. This phase is similar to the regret minimization phase of CSB-ST except that superarm to play is selected by Oracle that uses $KL(\hat\bmu(t), \hat\btheta, Q)$ to identify the arms where the learner has to allocate no resources.

\subsubsection{Regret Upper Bound}
\label{sssec:differentThetaRegretBounds}
Let $\nabla_{\ba}$ and $\nabla_m$ be defined as in Section \ref{sssec:sameThetaRegretBounds}. Let $\gamma>0$, $S_{\ba}=\{i:a_i < \theta_i\}$ for any feasible allocation $\ba$ and $k^\star = |S_{\ba^\star}|$. We redefine $W_T = \log (KT\log_2(\lceil 1 + 1/\gamma\rceil))/\log(1/(1-\epsilon))$.
\begin{restatable}{thm}{regretDiffThreshold}
	\label{thm:regretDiffThreshold}
	Let $(\bmu,\btheta, Q)\in \PCSB$ such that $\gamma>0$, $\mu_1\geq \epsilon$, and $T>W_T{\log_2 (\ceil{1 +  {1}/{\gamma}})}$. Set $\delta=1/T$ in \ref{alg:CSB-DT}. Then the expected regret of \ref{alg:CSB-DT} is upper bound as 
	\begin{align*}
		\EE{\Regret_T} \le &\left(\frac{KW_T{\log_2 \left(\ceil{1+1/\gamma} \right)}}{\max\{1, \floor{Q}\}}\right)\nabla_m  + \left(\sum_{i \in [K]} \max\limits_{S_{\ba}:i\in S_{\ba}}\frac{8|S_{\ba}|\log {T}}{\nabla_{\ba} - 2(k^\star{}^2 + 2)\eta}\right) +\\ 
		&\qquad \qquad \left(\frac{K(K-Q)^2}{\eta^2} +\frac{8\alpha_1}{\eta^2}\left(\frac{4}{\eta^2} + 1\right)^{k^\star} \log\frac{k^\star}{\eta^2} + 3K\right)\nabla_m
	\end{align*}
	for any $\eta$ such that $\forall \ba \in \A, \nabla_{\ba} > 2(k^\star{}^2+2)\eta$ and a problem independent constant $\alpha_1$. 
\end{restatable}

The first term of expected regret is due to the threshold estimation phase. Threshold estimation takes $T_{\theta_d}$ rounds to complete, and $\nabla_m$ is the maximum regret that can be incurred in any round. Then the maximum regret due to threshold estimation is bounded by $T_{\theta_d}\nabla_m$. The remaining terms correspond to the regret due to the regret minimization phase. Further, the expected regret of \ref{alg:CSB-DT} can be shown be $\EE{\Regret_T}\leq O(K\log T/\nabla_{\min})$, where $\nabla_{\min}$ is the minimum gap between the mean loss of optimal allocation and any non-optimal allocation.

	\section{Experiments}
	\label{sec:experiments}

We ran the computer simulations to evaluate the empirical performance of proposed algorithms. Our simulations involved two synthetically generated instances. In Instance 1, the threshold is the same for all arm, whereas Instance 2, it varies across arms. The details of the instances are as follows:

\textbf{Instance 1 (Identical Threshold):} It has $K = 20, Q=7, \theta_c=0.7, \delta=0.1,\epsilon=0.1$ and $T=10000$. The loss of arm $i$ is Bernoulli distribution with parameter $0.25 + (i-1)/50$. 
\\
\textbf{Instance 2 (Different Thresholds):} It has $K = 5, Q=2, \delta=0.1,\epsilon=0.1,\gamma=10^{-3}$ and $T=5000$. The mean loss vector is $\bmu = [0.9, 0.89,0.87,0.58,0.3]$ and corresponding  threshold vector is $\btheta=[0.7,0.7,0.7,0.6,0.35]$. The loss of arm $i$ is Bernoulli distributed with parameter $\mu_i$.

For Instance 1, we have only varied the number of resource $Q$, and regret of \ref{alg:CSB-ST} is shown in \cref{fig:QSameTheta}. We observe that when resources are small, the learner can allocate resources to a few arms but observes loss from more arms. On the other hand, when resources are more, the learner allocates resources to more arms and observes loss from fewer arms. Thus as resources increase, we move from semi-bandit feedback to bandit feedback and hence regret increase with the resources. Next, we have only varied $\theta_c$ in Instance 1, and the regret of \ref{alg:CSB-ST} is shown in \cref{fig:TSameTheta}. Similar trends are observed as the decrease in threshold leads to an increase in the number of arms that can be allocated resources and vice-versa. Therefore the amount of feedback decreases as the threshold decreases and leads to more regret. We repeated the experiment 100 times and plotted the regret with a 95\% confidence interval (the vertical line on each curve shows the confidence interval). The empirical results validate sub-linear bounds for our algorithms. 

\vspace{-2.5mm}
\begin{figure}[H]
	\centering
	\scriptsize
	\begin{minipage}[b]{0.309\textwidth}
		\captionsetup{justification=centering}
		\captionsetup{font=scriptsize, labelfont=bf}		\includegraphics[width=\linewidth]{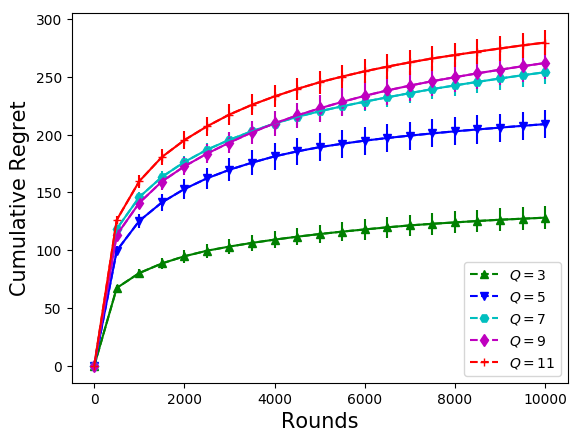}
		\caption{Varying amount of resources.}
		\label{fig:QSameTheta}
	\end{minipage}\qquad
	\begin{minipage}[b]{0.309\textwidth}
		\captionsetup{justification=centering}
		\captionsetup{font=scriptsize, labelfont=bf}
		\includegraphics[width=\linewidth]{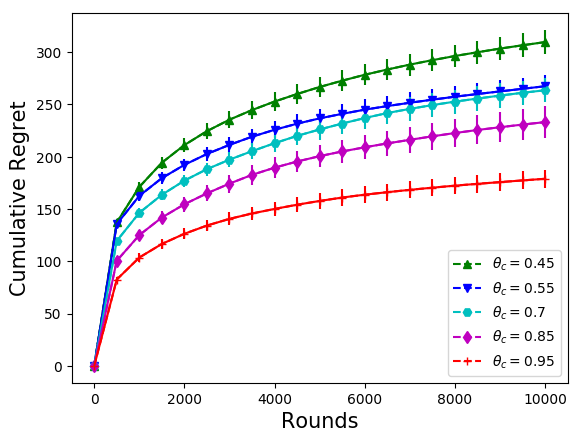}
		\caption{Varying value of same threshold.}
		\label{fig:TSameTheta}
	\end{minipage}\qquad
	\begin{minipage}[b]{0.309\textwidth}
		\captionsetup{justification=centering}
		\captionsetup{font=scriptsize, labelfont=bf}
		\includegraphics[width=\linewidth]{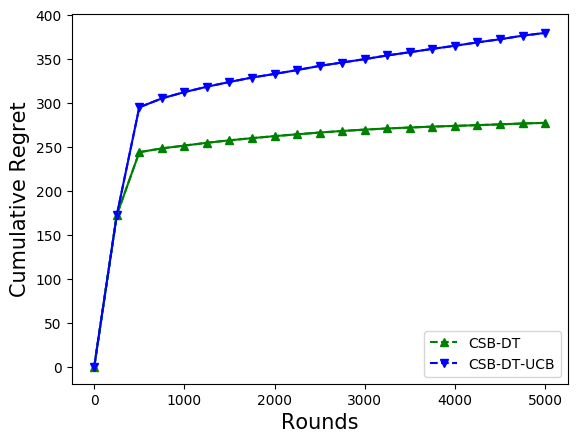}
		\caption{UCB and TS based Algorithms.}
		\label{fig:DiffTheta}
	\end{minipage}
\end{figure}
\vspace{-4mm}

We also compare the performance of \ref{alg:CSB-DT} against the CSB-DT-UCB algorithm, which uses the UCB type index as used in the SDCB algorithm \citep{NIPS16_chen2016combinatorial} on Instance 2. As shown in \cref{fig:DiffTheta}, as expected, Thompson Sampling (TS) based \ref{alg:CSB-DT} outperforms its UCB based counterpart CSB-DT-UCB. The pseudo-code of CSB-DT-UCB is given in the supplementary material.

	\section{Conclusion and Future Extensions}
	\label{sec:conclusion}

In this work, we proposed a novel framework for resource allocation problems using a variant of semi-bandits named censored semi-bandits. In our setup, loss observed from an arm depends on the amount of resource allocated, and hence, the loss can be censored. We consider a threshold-based model where loss from an arm is observed when allocated resource is below a threshold. The goal is to assign a given resource to arms such that total expected loss is minimized. We considered two variants of the problem, depending on whether or not the thresholds are the same across the arms. For the variant where thresholds are the same across the arms, we established that it is equivalent to the Multiple-Play Multi-Armed Bandit problem. For the second variant where threshold can depend on the arm, we established that it is equivalent to a more general Combinatorial Semi-Bandit problem. Exploiting these equivalences, we developed algorithms that enjoy optimal performance guarantees. 

We decoupled the problem of threshold and mean loss estimation. It would be interesting to explore if this can be done jointly, leading to better performance guarantees. Another new extension of work is to relax the assumptions that mean losses are strictly positive, and time horizon $T$ is known.

	\newpage
	\section*{Acknowledgments}
	    Arun Verma would like to thank travel support from Google and NeurIPS. Manjesh K. Hanawal would like to thank the support from INSPIRE faculty fellowships from DST, Government of India, SEED grant (16IRCCSG010) from IIT Bombay, and Early Career Research (ECR) Award from SERB. Initial discussions of this work were done when Raman Sankaran was at Conduent Labs India.
	
	\bibliographystyle{unsrtnat} 
	\bibliography{ref}

\begin{thebibliography}{26}
\providecommand{\natexlab}[1]{#1}
\providecommand{\url}[1]{\texttt{#1}}
\expandafter\ifx\csname urlstyle\endcsname\relax
  \providecommand{\doi}[1]{doi: #1}\else
  \providecommand{\doi}{doi: \begingroup \urlstyle{rm}\Url}\fi

\bibitem[Curtin et~al.(2010)Curtin, Hayslett-McCall, and
  Qiu]{NSE10_curtin2010determining}
Kevin~M Curtin, Karen Hayslett-McCall, and Fang Qiu.
\newblock Determining optimal police patrol areas with maximal covering and
  backup covering location models.
\newblock \emph{Networks and Spatial Economics}, 10\penalty0 (1):\penalty0
  125--145, 2010.

\bibitem[Adler et~al.(2014)Adler, Hakkert, Kornbluth, Raviv, and
  Sher]{AOR14_adler2014location}
Nicole Adler, Alfred~Shalom Hakkert, Jonathan Kornbluth, Tal Raviv, and Mali
  Sher.
\newblock Location-allocation models for traffic police patrol vehicles on an
  interurban network.
\newblock \emph{Annals of Operations Research}, 221\penalty0 (1):\penalty0
  9--31, 2014.

\bibitem[Rosenfeld and Kraus(2017)]{IJCAI17_rosenfeld2017security}
Ariel Rosenfeld and Sarit Kraus.
\newblock When security games hit traffic: Optimal traffic enforcement under
  one sided uncertainty.
\newblock In \emph{IJCAI}, pages 3814--3822, 2017.

\bibitem[Nguyen et~al.(2016)Nguyen, Sinha, Gholami, Plumptre, Joppa, Tambe,
  Driciru, Wanyama, Rwetsiba, Critchlow, et~al.]{AAMAS16_nguyen2016capture}
Thanh~H Nguyen, Arunesh Sinha, Shahrzad Gholami, Andrew Plumptre, Lucas Joppa,
  Milind Tambe, Margaret Driciru, Fred Wanyama, Aggrey Rwetsiba, Rob Critchlow,
  et~al.
\newblock Capture: A new predictive anti-poaching tool for wildlife protection.
\newblock In \emph{Proceedings of the 2016 International Conference on
  Autonomous Agents \& Multiagent Systems}, pages 767--775, 2016.

\bibitem[Gholami et~al.(2018)Gholami, Mc~Carthy, Dilkina, Plumptre, Tambe,
  Driciru, Wanyama, Rwetsiba, Nsubaga, Mabonga,
  et~al.]{AAMAS18_gholami2018adversary}
Shahrzad Gholami, Sara Mc~Carthy, Bistra Dilkina, Andrew Plumptre, Milind
  Tambe, Margaret Driciru, Fred Wanyama, Aggrey Rwetsiba, Mustapha Nsubaga,
  Joshua Mabonga, et~al.
\newblock Adversary models account for imperfect crime data: Forecasting and
  planning against real-world poachers.
\newblock In \emph{Proceedings of the 17th International Conference on
  Autonomous Agents and MultiAgent Systems}, pages 823--831, 2018.

\bibitem[Abernethy et~al.(2016)Abernethy, Amin, and
  Zhu]{NIPS16_abernethy2016threshold}
Jacob~D Abernethy, Kareem Amin, and Ruihao Zhu.
\newblock Threshold bandits, with and without censored feedback.
\newblock In \emph{Advances In Neural Information Processing Systems}, pages
  4889--4897, 2016.

\bibitem[Lattimore et~al.(2014)Lattimore, Crammer, and
  Szepesv{\'a}ri]{UAI14_lattimore2014optimal}
Tor Lattimore, Koby Crammer, and Csaba Szepesv{\'a}ri.
\newblock Optimal resource allocation with semi-bandit feedback.
\newblock In \emph{Proceedings of the Thirtieth Conference on Uncertainty in
  Artificial Intelligence}, pages 477--486. AUAI Press, 2014.

\bibitem[Zhang et~al.(2016)Zhang, Bucarey, Mukhopadhyay, Sinha, Qian,
  Vorobeychik, and Tambe]{AAMAS16_zhang2016using}
Chao Zhang, Victor Bucarey, Ayan Mukhopadhyay, Arunesh Sinha, Yundi Qian,
  Yevgeniy Vorobeychik, and Milind Tambe.
\newblock Using abstractions to solve opportunistic crime security games at
  scale.
\newblock In \emph{Proceedings of the 2016 International Conference on
  Autonomous Agents \& Multiagent Systems}, pages 196--204, 2016.

\bibitem[Sinha et~al.(2018)Sinha, Fang, An, Kiekintveld, and
  Tambe]{IJCAI18_sinha2018stackelberg}
Arunesh Sinha, Fei Fang, Bo~An, Christopher Kiekintveld, and Milind Tambe.
\newblock Stackelberg security games: Looking beyond a decade of success.
\newblock In \emph{IJCAI}, pages 5494--5501, 2018.

\bibitem[Zhang et~al.(2015)Zhang, Sinha, and Tambe]{AAMAS15_zhang2015keeping}
Chao Zhang, Arunesh Sinha, and Milind Tambe.
\newblock Keeping pace with criminals: Designing patrol allocation against
  adaptive opportunistic criminals.
\newblock In \emph{Proceedings of the 2015 international conference on
  Autonomous agents and multiagent systems}, pages 1351--1359, 2015.

\bibitem[Badanidiyuru et~al.(2018)Badanidiyuru, Kleinberg, and
  Slivkins]{JACM18_badanidiyuru2018bandits}
Ashwinkumar Badanidiyuru, Robert Kleinberg, and Aleksandrs Slivkins.
\newblock Bandits with knapsacks.
\newblock \emph{Journal of the ACM (JACM)}, 65\penalty0 (3):\penalty0 13, 2018.

\bibitem[Jain and Jamieson(2018)]{ICML18_jain2018firing}
Lalit Jain and Kevin Jamieson.
\newblock Firing bandits: Optimizing crowdfunding.
\newblock In \emph{International Conference on Machine Learning}, pages
  2211--2219, 2018.

\bibitem[Lattimore et~al.(2015)Lattimore, Crammer, and
  Szepesv{\'a}ri]{NIPS15_lattimore2015linear}
Tor Lattimore, Koby Crammer, and Csaba Szepesv{\'a}ri.
\newblock Linear multi-resource allocation with semi-bandit feedback.
\newblock In \emph{Advances in Neural Information Processing Systems}, pages
  964--972, 2015.

\bibitem[Dagan and Koby(2018)]{ALT18_dagan18a}
Yuval Dagan and Crammer Koby.
\newblock A better resource allocation algorithm with semi-bandit feedback.
\newblock In \emph{Proceedings of Algorithmic Learning Theory}, pages 268--320,
  2018.

\bibitem[Cesa-Bianchi and Lugosi(2012)]{JCSS12_cesa2012combinatorial}
Nicolo Cesa-Bianchi and G{\'a}bor Lugosi.
\newblock Combinatorial bandits.
\newblock \emph{Journal of Computer and System Sciences}, 78\penalty0
  (5):\penalty0 1404--1422, 2012.

\bibitem[Chen et~al.(2013)Chen, Wang, and Yuan]{ICML13_chen2013combinatorial}
Wei Chen, Yajun Wang, and Yang Yuan.
\newblock Combinatorial multi-armed bandit: General framework and applications.
\newblock In \emph{International Conference on Machine Learning}, pages
  151--159, 2013.

\bibitem[Rajkumar and Agarwal(2014)]{NIPS14_rajkumar2014online}
Arun Rajkumar and Shivani Agarwal.
\newblock Online decision-making in general combinatorial spaces.
\newblock In \emph{Advances in Neural Information Processing Systems}, pages
  3482--3490, 2014.

\bibitem[Combes et~al.(2015)Combes, Shahi, Proutiere,
  et~al.]{NIPS15_combes2015combinatorial}
Richard Combes, Mohammad Sadegh Talebi~Mazraeh Shahi, Alexandre Proutiere,
  et~al.
\newblock Combinatorial bandits revisited.
\newblock In \emph{Advances in Neural Information Processing Systems}, pages
  2116--2124, 2015.

\bibitem[Chen et~al.(2016)Chen, Hu, Li, Li, Liu, and
  Lu]{NIPS16_chen2016combinatorial}
Wei Chen, Wei Hu, Fu~Li, Jian Li, Yu~Liu, and Pinyan Lu.
\newblock Combinatorial multi-armed bandit with general reward functions.
\newblock In \emph{Advances in Neural Information Processing Systems}, pages
  1659--1667, 2016.

\bibitem[Wang and Chen(2018)]{ICML18_wang2018thompson}
Siwei Wang and Wei Chen.
\newblock Thompson sampling for combinatorial semi-bandits.
\newblock In \emph{International Conference on Machine Learning}, pages
  5101--5109, 2018.

\bibitem[Komiyama et~al.(2015)Komiyama, Honda, and
  Nakagawa]{ICML15_komiyama2015optimal}
Junpei Komiyama, Junya Honda, and Hiroshi Nakagawa.
\newblock Optimal regret analysis of thompson sampling in stochastic
  multi-armed bandit problem with multiple plays.
\newblock In \emph{International Conference on Machine Learning}, pages
  1152--1161, 2015.

\bibitem[Cesa-Bianchi et~al.(2006)Cesa-Bianchi, Lugosi, and
  Stoltz]{MOR06_cesa2006regret}
Nicolo Cesa-Bianchi, G{\'a}bor Lugosi, and Gilles Stoltz.
\newblock Regret minimization under partial monitoring.
\newblock \emph{Mathematics of Operations Research}, 31\penalty0 (3):\penalty0
  562--580, 2006.

\bibitem[Bart{\'o}k and Szepesv{\'a}ri(2012)]{ICML12_bartok2012partial}
G{\'a}bor Bart{\'o}k and Csaba Szepesv{\'a}ri.
\newblock Partial monitoring with side information.
\newblock In \emph{International Conference on Algorithmic Learning Theory},
  pages 305--319. Springer, 2012.

\bibitem[Bart{\'o}k et~al.(2014)Bart{\'o}k, Foster, P{\'a}l, Rakhlin, and
  Szepesv{\'a}ri]{MOR14_bartok2014partial}
G{\'a}bor Bart{\'o}k, Dean~P Foster, D{\'a}vid P{\'a}l, Alexander Rakhlin, and
  Csaba Szepesv{\'a}ri.
\newblock Partial monitoring---classification, regret bounds, and algorithms.
\newblock \emph{Mathematics of Operations Research}, 39\penalty0 (4):\penalty0
  967--997, 2014.

\bibitem[Anantharam et~al.(1987)Anantharam, Varaiya, and
  Walrand]{TAC1987_MultiPlayBandits_Anatharam}
V.~Anantharam, P.~Varaiya, and J.~Walrand.
\newblock Asymptotically efficient allocation rules for the multiarmed bandit
  problem with multiple plays- part {I}.
\newblock \emph{IEEE Transactions on Automatic Control}, 32\penalty0
  (11):\penalty0 968--976, 1987.

\bibitem[Hifi and Mhalla(2013)]{DO13_hifi2013sensitivity}
Mhand Hifi and Hedi Mhalla.
\newblock Sensitivity analysis to perturbations of the weight of a subset of
  items: The knapsack case study.
\newblock \emph{Discrete Optimization}, 10\penalty0 (4):\penalty0 320--330,
  2013.

\end{thebibliography}

	\ifsup
    	\newpage
    	\noindent\rule{\linewidth}{4pt} 
    	\vspace{1.5mm}
    	
    	\centerline{\Large \bf Supplementary Material for `Censored Semi-Bandits: } \vspace{3mm} \centerline{\Large\bf A Framework for Resource Allocation with Censored Feedback'
    	}
    	\vspace{0.5mm}
    	\hrulefill \\
    
    	\appendix

\section{Proofs related to Section \ref{sec:same_theta}}

\subsection{Proof of Lemma \ref{lem:thetaSet}}
\ThetaSet*
\begin{proof}
	The case $\floor{Q/\theta_c} \geq K$ is trivial. We consider the case $\floor{Q/\theta_c} < K$. By definition $M=\min\{\floor{Q/\theta_c},K\}$. We have $M \le Q/\theta_c$ and $\theta_c \le Q/M \doteq \hat{\theta}_c$. Hence $\hat{\theta}_c \ge \theta_c$. Therefore $\hat{\theta}_c$ fraction of resource allocation at a location has same reduction in the mean loss as $\theta_c$. Further, in both the instances $(\bmu, \theta_c, Q)$ and $(\bmu, \hat{\theta}_c, Q)$ the optimal allocations incur no loss from the bottom-$M$ and the same amount of loss from the top $K-M$ arms. Hence the mean loss reduction for both the instances is same. This completes the proof of first part. As $M \in \{1, \ldots, K\}$ and $\hat{\theta}_c \le 1$, the possible value of $\hat{\theta}_c$ is only one of element in the set $\Theta = \{  Q/K, Q/(K-1), \cdots, \min\{1, Q\}\}$. 
\end{proof}

\subsection{Proof of Lemma \ref{lem:sameThresholdEstRounds}}
\sameThresholdEstRounds*
\begin{proof}
	When $\hat{\theta}_c < \theta_c$, it can happen that no loss is observed for $W$ consecutive rounds that leads to incorrect estimation of $\theta_c$. We want to set $W$ such a way that the probability of occurring of such event is small. This probability is bounded as follows:
	\begin{align*}
	&\Prob{\text{No loss is observed on $Q/\hat{\theta}_c$ arms for $W$ consecutive rounds at $\hat{\theta}_c<\theta_c$ (underestimate)}} \\
	&\qquad = \prod_{i \le Q/\hat{\theta}_i} (1 - \mu_i)^W \hspace{5mm}\text{\big(as $(1 - \mu_i)$ is the probability of not observing loss at location $i$\big)}\\
	&\qquad \le \prod_{i \le M} (1 - \mu_i)^W \hspace{8mm} \mbox{\big(as $\hat{\theta}_c < \theta_c \implies M \le Q/\hat{\theta}_c$\big)} \\
	&\qquad \le \prod_{i \le M} (1 - \epsilon)^W = (1 - \epsilon)^{MW} \le (1 - \epsilon)^{QW} \hspace{5mm}\text{\big(as $M\ge Q$\big)}
	\end{align*}
	Since we are using binary search, the algorithm goes through at most $\log_2(|\Theta|)$ underestimates of $\theta_c$. Let $I$ denote the indices of these underestimates in $\Theta$
	\begin{align*}
	&\Prob{\text{No loss is observed for consecutive $W$ rounds at any  underestimate of $\theta_c$}} \\
	&\qquad \le \sum_{i \in I}\Prob{\text{No loss is observed for consecutive $W$ rounds at the underestimate $\Theta(i)$}} \\
	&\qquad \le (1 - \epsilon)^{QW}\log_2(|\Theta|)
	\end{align*}
	
	We want to bound the probability of making mistake by $\delta$. We get,
	\begin{equation*}
	(1 - \epsilon)^{QW}\log_2(|\Theta|) \le \delta \implies (1 - \epsilon)^{QW} \le \delta/\log_2(|\Theta|)
	\end{equation*}
	Taking log both side, we get
	\begin{align*}
	QW\log(1 - \epsilon) \le \log(\delta/\log_2(|\Theta|)) &\implies QW\log\left(\frac{1}{1 - \epsilon}\right) \ge \log(\log_2(|\Theta|)/\delta)\\ &\implies W \ge \frac{\log(\log_2(|\Theta|)/\delta)}{Q\log\left(\frac{1}{1 - \epsilon}\right) }
	\end{align*}
	We set
	\begin{equation}
	W=\frac{\log(\log_2(|\Theta|)/\delta)}{Q\log\left(\frac{1}{1 - \epsilon}\right) }
	\end{equation}
	Hence, the minimum rounds needed to find $\hat{\theta}_c$ with probability at least $1-\delta$ is $W\log_2(|\Theta|)$.
\end{proof}

\subsection{Proof of Proposition \ref{prop:RegretEquiST}}
\RegretEquiST*
\begin{proof}
	Let $\pi$ be a policy on $P:=(\bmu,\theta_c,Q) \in \PCSB^c$. The regret of policy $\pi$ on $P$ is given by
	\[\Regret_T(\pi,P)=\sum_{t=1}^T \left(\sum_{i=1}^K \mu_i\one{a_{t,i} < \theta_c} -\sum_{i=1}^K \mu_i\one{a^\star_i < \theta_c}\right), \]
	where $\ba^\star$ is the optimal allocation for $P$. Consider $f(P)=(\bmu,m) \in \PMP$ where $\bmu$ is the same as in $P$ and $m=K-M$, where $M=\min\{\lfloor Q/\theta_c\rfloor,K\}$. The regret of policy $\pi^\prime$ on $f(P)$ is given by 
	\[\Regret_T(\pi^\prime,f(P))=\sum_{t=1}^T\left(\sum_{i \in S_t} \mu_i - \sum_{i=1}^{K-M}\mu_i\right),\]
	where $S_t$ is the superarm played in round $t$. Recall the ordering $\mu_1 \le \mu_2 \le \ldots \le \mu_K$. It is clear that $\sum_{i=1}^K \mu_i\one{a^\star_i< \theta_c}=\sum_{i=1}^{K-M}\mu_i$. 
	Let $L_t$ be the set of arms where no resources are allocated by $\pi$ in round $t$. Since, loss only incurred from arms in the set $L_t$, we have $\sum_{i=1}^K \mu_i\one{a_{t,i} < \theta_c}=\sum_{i \in L_t}\mu_i$. From the definition of $\pi^\prime$, notice that in round $t$, $\pi^\prime$ selects superarm $S_t=L_t$, i.e., set of arms returned by $\pi$ for which no resourced are applied. Hence $\sum_{i=1}^K \mu_i\one{a_{t,i}< \theta_c}=\sum_{i \in S_t} \mu_i $. This establishes the regret of $\pi$ on $P$ is same as regret of $\pi^\prime$ on $f(P)$ and we get $\Regret(\PMP)\leq \Regret (\PCSB^c)$.
	Similarly, we can establish the other direction of the proposition and get $\Regret(\PCSB^c)\leq \Regret (\PMP)$. Thus we conclude  $\Regret(\PCSB^c)= \Regret (\PMP)$.
\end{proof}

\subsection{Proof of Theorem  \ref{thm:regretSameThreshold}}
Let $M, \nabla_m$ and $W$ be defined as in Section \ref{sssec:sameThetaRegretBounds}. We use the following results to prove the theorem.
\begin{thm}
	\label{thm:MPRegret}
	Let $\hat\theta_c$ be allocation equivalent to $\theta_c$ for instance $(\bmu,\theta_c,Q)$. Then, the expected  regret of the regret minimization phase of \ref{alg:CSB-ST} for $T$ rounds is upper bound as     
	\begin{align}
	\EE{\Regret_T} \le O\left((\log T)^{{2}/{3}}\right) + \sum_{i \in [K]\setminus [K-M]} \frac{(\mu_i-\mu_{K-M} )\log {T}}{d( \mu_{K-M},\mu_i)}.
	\end{align}
\end{thm}

\begin{proof}
	As $\hat\theta_c$ be allocation equivalent to $\theta_c$, the instances $(\bmu,\theta_c,Q)$ and $(\bmu,\hat{\theta}_c,Q)$ have same minimum loss. Also, by the equivalence established is Proposition \ref{prop:RegretEquiST}, the regret minimization phase of \ref{alg:CSB-ST} is solving a MP-MAB instance. Then we can directly apply Theorem 1 of \cite{ICML15_komiyama2015optimal} to obtain the regret bounds by setting $k=K-M$ and noting that we are in the loss setting and a mistake happens when a arm $i \in [K]\setminus[K-M]$ is in selected superarm.
\end{proof}
\begin{thm}
	\label{thm:regretHighConf}
	With probability at least $1-\delta$, the expected cumulative regret of \ref{alg:CSB-ST} is upper bounded as
	\begin{equation*}
	\EE{\Regret_T} \le W\log_2{(|\Theta|)} \nabla_m  + O\left((\log T)^{{2}/{3}}\right) + \sum_{i \in [K]\setminus [K-M]} \frac{(\mu_i-\mu_{K-M} )\log {T}}{d( \mu_{K-M},\mu_i)}.
	\end{equation*}
\end{thm}

\begin{proof}
	\ref{alg:CSB-ST} has two phases: threshold estimation and loss minimization. Threshold estimation runs for at most $W\log_2{(|\Theta|)}$ rounds and returns an allocation equivalent threshold with probability at least $1-\delta$. The maximum regret incurred in this phase is $W\log_2{(|\Theta|)} \nabla_m$. Given that the threshold estimated in the threshold estimation phase is correct, the regret incurred in the regret minimization phase is given by Theorem \ref{thm:MPRegret}. Thus the expected regret of \ref{alg:CSB-ST} is given by the sum of regret incurred in these two phases and holds with probability at least $1-\delta$.  
\end{proof}

\regretSameThreshold*
\begin{proof}
	The bound follows from Theorem \ref{thm:regretHighConf} by setting $\delta=1/T$ and unconditioning the expected regret obtained in the regret minimization phase of  \ref{alg:CSB-ST} .
\end{proof}

\section{Proofs related to Section \ref{ssec:different_theta}}

\subsection{Proof of Proposition \ref{prop:diffThetaOptiSoln}}
\diffThetaOptiSoln*
\begin{proof}
	Assigning $\theta_i$ fraction of resources to the arm $i$ reduces the total mean loss by amount $\mu_i$. Our goal is to allocate resources such that total mean loss is minimize i.e., $\min\limits_{\ba \in \A}\sum_{i\in[K]}\mu_i\one{a_i < \theta_i}$. Note that the maximization version of same optimization problem is $\max\limits_{\ba \in \A}\sum_{i\in[K]}\mu_i\one{a_i \ge \theta_i}$ which is same as solving a 0-1 knapsack with capacity $Q$ where item $i$ has value $\mu_i$ and weight $\theta_i$. 
\end{proof}

\subsection{Proof of Lemma \ref{lem:diffTheteEst}}
\diffTheteEst*
\begin{proof}
	Let $L^\star = \left\{i: a_i^\star < \theta_i\right\}$ and $r = Q - \sum_{i: a_i^\star \ge \theta_i}\theta_i$. If resource $r$ is allocated to any arm $i \in L^\star$, minimum value of mean loss will not change as $r < \min_{i \in L^\star} \theta_i$. If we can allocate $\gamma=r/K$ fraction of $r$ to each arm $i \in K$, the minimum mean loss still remains same. If estimated threshold of every arm $i \in K$ lies in $[\theta_i, \theta_i+\gamma]$ then using Theorem 3.2 of \cite{DO13_hifi2013sensitivity}, $KP(\bmu,\btheta, Q)$ and $KP(\bmu, \hat\btheta, Q)$ has the same optimal solution because of having the same mean loss for both the problem instances.
\end{proof}

\subsection{Proof of Lemma \ref{lem:MultiTheta}}
\MultiTheta*
\begin{proof}
	For any arm $i \in [K]$, we want $\hat{\theta}_i \in [\theta_i, \theta_i+\gamma]$. As $\theta_i \in (0,1]$, we can divide interval $[0,1]$ into a discrete set $\Theta \doteq \left\{0, \gamma, 2\gamma, \ldots, 1\right\}$ and note that $|\Theta| = \ceil{1+ {1}/{\gamma}}$. As search space is reduced by half in each change of $\hat\theta_i$, the maximum change in $\hat\theta_i$ is upper bounded by $\log_2|\Theta|$ to make sure that $\hat{\theta}_i \in [\theta_i, \theta_i+\gamma]$. When $\hat{\theta}_i$ is underestimated and no loss is observed for consecutive $W$ rounds, a mistake is happened by assuming that current allocation is overestimated. We set $W$ such that the probability of estimating wrong $\hat{\theta}_i$ is small and bounded as follows:
	\begin{align*}
	&\Prob{\text{No loss is observed for consecutive $W$ rounds when $\hat{\theta}_i$ is underestimated}} \\
	&\qquad = (1 - \mu_i)^W \hspace{5mm}\text{\big(as $(1 - \mu_i)$ is the probability of not observing loss at arm $i$\big)}\\
	&\qquad \le (1 - \epsilon)^W \hspace{6.8mm} \text{\big(since $\forall i \in [K]: \mu_i > \epsilon$\big)}
	\end{align*}
	Since we are doing binary search, the algorithm goes through at most $\log_2(|\Theta|)$ underestimates of $\theta_i$. Let $I$ denote the indices of these underestimates in $\Theta$
	\begin{align*}
	&\Prob{\text{No loss is observed for consecutive $W$ rounds when  $\hat{\theta}_i$ is underestimated}} \\
	&\qquad \le \sum_{i\in I}\Prob{\text{No loss is observed for consecutive $W$ rounds when  $\hat{\theta}_i$ is underestimated}} \\
	&\qquad \le (1 - \epsilon)^{W}\log_2(|\Theta|)
	\end{align*}
	Next, we will bound the probability of making mistake for any of the arm. That is given by
	\begin{align*}
	&\Prob{\exists i  \in [K], \hat\theta_i \in \Theta: \text{No loss is observed for consecutive $W$ rounds when  $\hat{\theta}_i$ is underestimated}} \\
	&~~ \le \sum_{i=1}^{K}\Prob{\exists \hat\theta_i \in \Theta: \text{No loss is observed for consecutive $W$ rounds when  $\hat{\theta}_i$ is underestimated}}  \\
	&~~ \le K(1 - \epsilon)^{W}\log_2(|\Theta|)
	\end{align*}
	
	We want to bound the above probability of making a mistake by $\delta$ for all arms. We get,
	\begin{equation*}
	K (1 - \epsilon)^{W}\log_2(|\Theta|) \le \delta \implies (1 - \epsilon)^{W} \le \delta/K \log_2(|\Theta|)
	\end{equation*}
	Taking log both side, we get
	\begin{align*}
	W\log(1 - \epsilon) \le \log(\delta/K \log_2(|\Theta|)) &\implies W\log\left({1}/{(1 - \epsilon)}\right) \ge \log(K \log_2(|\Theta|)/\delta)\\
	&\implies W \ge \frac{\log(K \log_2(|\Theta|)/\delta)}{\log\left({1}/{(1 - \epsilon)}\right) }
	\end{align*}
	We set 
	\begin{equation}
	W = \frac{\log(K \log_2(|\Theta|)/\delta)}{\log\left({1}/{(1 - \epsilon)}\right) }
	\end{equation}
	
	Therefore, the minimum rounds needed for an arm $i$ to correctly find $\hat{\theta}_i$ with probability at least $1-\delta$ is upper bounded by $W\log_2(|\Theta|)$. \ref{alg:CSB-DT} can simultaneously estimate threshold for at least $\max\{1, \floor{Q}\}$ arms by exploiting the fact that $\hat\theta_i \le 1$. The threshold for $K$ arms needs to be estimate, hence, minimum rounds needed to correctly find all $\hat\theta_i \in [\theta_i, \theta_i+\gamma]$ with probability at least $1-\delta$ is $KW\log_2(|\Theta|)/\max\{1, \floor{Q}\}$.
\end{proof}

\subsection{ Proof of Proposition \ref{prop:MultiThetaEquivalence}}
\subsection*{Equivalence of CSB with different thresholds and Combinatorial Semi-Bandit }
In stochastic Combinatorial Semi-Bandits (CoSB), a learner can play a subset from $K$ arms in each round, also known as superarm, and observes the loss from each arm played \citep{ICML13_chen2013combinatorial,NIPS16_chen2016combinatorial,ICML18_wang2018thompson}. The size of superarm can vary, and the mean loss of a superarm only depends on the means of its constituent arms. The goal of the learner is to select a superarm that has the smallest loss. A policy in CoSB selects a superarm in each round based on the past information. The performance of a policy is measured in terms of regret defined as the difference between cumulative loss incurred by the policy and that incurred by playing an optimal superarm in each round. Let $(\bmu, \mathcal{I}) \in [0,1]^K \times 2^{[K]}$ denote an instance of CoSB where $\bmu$ denote the mean loss vector, and $\mathcal{I}$ denotes the set of superarms. Let $\PCSB^d \subset \PCSB$ denote the set of CSB instances with a different threshold for arms. For any  $(\bmu,\btheta,Q) \in \PCSB^d$ with $K$ arms and known threshold $\btheta$, let $(\bmu, \mathcal{I})$ be an instance of CoSB with $K$ arms and each arm has the same Bernoulli distribution as the corresponding arm in the CSB instance. Let $\PCoSB$ denote set of resulting CoSB problems and $g: \PCSB \rightarrow \PCoSB$ denote the above transformation.

Let $\pi$ be a policy on $\PCoSB$. $\pi$ can also be applied on any $(\bmu,\btheta,Q) \in \PCSB^d$ with known $\btheta$ to decide which set of arms are allocated resource as follows: in round $t$, let information $(L_1, Y_1, L_2,Y_2, \ldots, L_{t-1}, Y_{t-1})$ collected on an CSB instance, where $L_s$ is the set of arms where no resource is applied and $Y_s$ is the samples observed from these arms, is given to $\pi$ which returns a set $L_t$. Then all arms other than arms in $L_t$ are given resource equal to their  estimate threshold. Let this policy on $(\bmu,\btheta,Q) \in \PCSB^d$ is denoted as $\pi^\prime$. In a similar way a policy $\beta$ on $\PCSB$ can be adopted to yield a policy for $\PCoSB$ as follows: in round $t$, the information $(S_1, Y_1, S_2,Y_2, \ldots, S_{t-1}, S_{t-1})$, where $S_s$ is the superarm played in round $s$ and $Y_s$ is the associated loss observed from each arms in $S_s$, collected on an CoSB instance is given to $\pi$. Then $\pi$ returns a set $S_t$ where no resources has to applied. The superarm corresponding to $S_t$ is then played. Let this policy on $\PCoSB$ be denote as $\beta^\prime$. Note that when $\btheta$ is known, the mapping is invertible.

\MultiThetaEquivalence*
\begin{proof}
	Let $\pi$ be a policy on $P:=(\bmu,\btheta,Q) \in \PCSB^d$. The regret of policy $\pi$ on $P$ is given by
	\[\Regret_T(\pi,P)=\sum_{t=1}^T \left(\sum_{i=1}^K \mu_i\one{a_{t,i}< \theta_i} -\sum_{i=1}^K \mu_i\one{a^\star_i < \theta_i}\right), \]
	where $\ba^\star$ is the optimal allocation for $P$. $g(P)=(\bmu,\mathcal{I}) \in \PCoSB$ where $\bmu$ is the same as in $P$ and $\mathcal{I}$ contains all superarms (set of arms) for which resource allocation is feasible. The regret of policy $\pi^\prime$ on $g(P)$ is given by 
	\[\Regret_T(\pi^\prime,g(P))=\sum_{t=1}^T\big(l(S_t,\bmu) - l(S^\star,\bmu)\big)\]
	where $S_t$ is the superarm played in round $t$,  $S^\star$ is optimal superarm, and $l$ returns mean loss.
	Note that outcomes of $l(S,\bmu)$ only depends on mean loss of constituents arms of the superarm $S$. In our setting, $l(S,\bmu) = \sum_{i \in S}\mu_i$ where $S= \left\{i: a_i <\theta_i \right\}$ for allocation $\ba \in \A$. It is clear that $\sum_{i=1}^K \mu_i\one{a^\star_i < \theta_i}=l(S^\star, \bmu)$. Let $L_t$ be the set of arms where no resources is allocated by $\pi$ in round $t$. Since, loss only incurred from arms in the set $L_t$, we have $\sum_{i=1}^K \mu_i\one{a_{t,i}< \theta_c}=\sum_{i \in L_t}\mu_i$. From the definition of $\pi^\prime$, notice that in round $t$, $\pi^\prime$ selects superarm $S_t=L_t$, i.e., set of arms returned by $\pi$ for which no resourced are applied. Hence $\sum_{i=1}^K \mu_i\one{a_{t,i}< \theta_c}=\sum_{i \in S_t} \mu_i $. This establishes the regret of $\pi$ on $P$ is same as regret of $\pi^\prime$ on $g(P)$ and we get $\Regret(\PCoSB)\leq \Regret (\PCSB^d)$.
	Similarly, we can establish the other direction of the proposition and get $\Regret(\PCSB^d)\leq \Regret (\PCoSB)$. Thus we conclude  $\Regret(\PCSB^d)= \Regret (\PCoSB)$.
\end{proof}

\subsection{Proof of Theorem  \ref{thm:regretDiffThreshold}}
Let $\nabla_{\ba}$ and $\nabla_m$ be defined as in Section \ref{sssec:sameThetaRegretBounds}. Let $S_{\ba}=\{i:a_i < \theta_i\}$ for a feasible allocation $\ba$, $k^\star = |S_{\ba^\star}|$ and $W$ be same as in Section \ref{sssec:differentThetaRegretBounds}. Note that we will never be able to sample a $\hat\mu_i(t)$ to be precisely the true value $\mu_i$ using Beta distribution.  We need to consider the $\eta$-neighborhood of $\mu_i$, and such $\eta$ term is common in the analysis of most Thompson Sampling based algorithms (see \cite{ICML18_wang2018thompson} for more details). We need the following results to prove the theorem. 
\begin{thm}
	\label{thm:CTSRegret}
	Let $\hat\btheta$ be allocation equivalent to $\btheta$ for instance $(\bmu,\btheta,Q)$. Then, the expected regret of the regret minimization phase of \ref{alg:CSB-DT} in $T$ rounds is upper bound as $\left(\sum_{i \in [K]} \max\limits_{S_{\ba}:i\in S_{\ba}}\frac{8|S_{\ba}|\log {T}}{\nabla_{\ba} - 2(k^\star{}^2 + 2)\eta}\right) + \left(\frac{K(K-N)^2}{\eta^2} + \frac{8\alpha_1}{\eta^2}\left(\frac{4}{\eta^2} + 1\right)^{k^\star} \log\frac{k^\star}{\eta^2} + 3K\right)\nabla_m$ for any $\eta$ such that $\forall \ba \in \A, \nabla_{\ba}>2(k^\star{}^2+2)\eta$ and $\alpha_1$ is a problem independent constant. 
\end{thm}
\begin{proof}
	Once the allocation equivalent to $\btheta$ is known, the regret minimization problem is equivalent to solving a combinatorial semi-bandit problem (from \cref{prop:MultiThetaEquivalence}). The proof follows by verifying Assumptions $1-3$ in \citep{ICML18_wang2018thompson} for the combinatorial semi-bandit problem and applying their regret bounds. Assumption $1$ states that the mean of a superarm depends only on the means of its constituents arms (Assumption $1$) and distributions of the arms are independent (Assumptions $3$). It is clear that both of these assumptions hold for our case. We next proceed to verify Assumption $2$. For fix allocation $\ba\in \A$, the mean loss incurred from loss vector $\bmu$ is given by $l(S,\bmu)=\sum_{i \in S}\bmu_i$ where $S=\left\{i:a_i < \hat\theta_i\right\}$. For any loss vectors $\bmu$ and $\bmu^\prime$, we have
	\begin{align*}
	l(S, \bmu)-l(S, \bmu^\prime)&=\sum_{i \in S}(\mu_i - \mu_i^\prime) \\
	&= \sum_{i=1}^K  \one{a_i< \hat\theta_i}\left (\mu_i -\mu_i^\prime \right) \hspace{5mm} \text{$\Bigg($as $\sum_{i \in S}\mu_i = \sum_{i=1}^K \mu_i \one{a_i< \hat\theta_i}\Bigg)$}\\
	&\leq    \sum_{i=1}^K  \left (\mu_i -\mu_i^\prime \right)\leq    \sum_{i=1}^K   |\mu_i -\mu_i^\prime |
	= B\parallel \bmu- \bmu^\prime \parallel_1
	\end{align*}
	where $B=1$. Also, note that in the regret minimization phase, the allocation to each arm remains the same in each round ($\hat{\theta}_i$ is given to each arm $i \in [K]\setminus L_t$). Thus we are solving a combinatorial semi-bandit with parameter $B=1$ in the regret minimization phase. By applying Theorem $1$ in \cite{ICML18_wang2018thompson}, we get the desired bounds.
\end{proof}

\begin{thm}
	\label{thm:regretDiffThresholdHighConf}
	With probability at least $1-\delta$, the expected cumulative regret of \ref{alg:CSB-DT} is upper bound as $\EE{\Regret_T} \le \left(\frac{KW\log_2 \left(\ceil{ 1+1/\gamma} \right)}{\max\{1, Q\}}\right)\nabla_m  + \left(\sum_{i \in [K]} \max\limits_{S_{\ba}:i\in S_{\ba}}\frac{8|S_{\ba}|\log {T}}{\nabla_{\ba} - 2(k^\star{}^2 + 2)\eta}\right) + \left(\frac{K(K-N)^2}{\eta^2} + \frac{8\alpha_1}{\eta^2}\left(\frac{4}{\eta^2} + 1\right)^{k^\star} \log\frac{k^\star}{\eta^2} + 3K\right)\nabla_m.$
\end{thm}
\begin{proof}
	The threshold estimation phase runs for at most $KW{\log_2 \left(\ceil{ 1+1/\gamma} \right)}/\max\{1, Q\}$ rounds and finds an allocation equivalent threshold with probability $1-\delta$. The regret incurred by this phase is $(KW{\log_2 \left(\ceil{ 1+1/\gamma} \right)}/\max\{1, Q\})\nabla_m$ which form the first part of the bounds. Once an allocation equivalent threshold is found, the upper bound on expected regret incurred in the regret minimization phase is given by  Theorem \ref{thm:CTSRegret}. Thus the regret of \ref{alg:CSB-DT}  is given by sum of these two quantities holds with probability at least $(1-\delta)$.
\end{proof}

\regretDiffThreshold*
\begin{proof}
	The bound follows from Theorem \ref{thm:regretDiffThresholdHighConf} by setting $\delta=1/T$ and unconditioning the expected regret obtained in the regret minimization phase of \ref{alg:CSB-DT}.
\end{proof}

\section{Leftover details from Section \ref{sec:experiments}}
\subsection{Algorithm: \ref{alg:CSB-DT-UCB}}
In our experiments, we used an UCB index-based algorithm named \ref{alg:CSB-DT-UCB} for comparing cumulative regret with Thompson Sampling based algorithm \ref{alg:CSB-DT}.  \ref{alg:CSB-DT-UCB} works as follows: it keeps track of number of losses $(S_i)$ and number of observations $(N_i)$ for each arm $i \in [K]$. At round $t$, it computes lower bound of mean losses of all arms (Line $23$). We use the lower bound due to our loss setting. The remaining part of \ref{alg:CSB-DT-UCB} is same as \ref{alg:CSB-DT}.

\begin{algorithm}
	\renewcommand{\thealgorithm}{\bf CSB-DT-UCB}
	\floatname{algorithm}{}
	\caption{UCB based Algorithm for solving the CSB problem with Different Threshold}
	\label{alg:CSB-DT-UCB}
	\begin{algorithmic}[1]
		\Statex\textbf{Input:} $K, Q, \delta, \epsilon, \gamma$
		\vspace{1mm}
		\Statex \textbackslash\textbackslash \textbf{ Threshold Estimation Phase} \textbackslash\textbackslash
		\vspace{0.5mm}
		\State Initialize $\forall i \in [K]: \theta_{l,i}= 0, \theta_{u,i}= 1, \theta_{g,i}= 0, C_i =0$. 
		\State Set $T_{\theta_d}=0, W = {\log(K\log_2(\lceil 1 + 1/\gamma\rceil)/\delta)}/ {\log(1/(1-\epsilon))}$
		\State $\forall i \in [\floor{Q}]:$ allocate $\hat\theta_i = 0.5$ resource. $\forall j \in [\lfloor Q \rfloor +1,K]:$ allocate $\hat\theta_j =\frac{Q-\floor{Q}/2}{K-\floor{Q}}$ resources
		\While {$\theta_{g,i}= 0$ for any $i \in [K]$}
			\For{$i = 1, \ldots, K$} 
				\If {loss observe for arm $i$ with $\theta_{g,i}=0$ and $\theta_{l,i} < \hat\theta_i$}
					\State Set $\theta_{l,i}= \hat\theta_i, \hat\theta_i = (\theta_{u,i}+ \theta_{l,i})/2, C_i=0$. If available allocate resource $\hat\theta_i$
				\Else
					\State If allocated resources is $\hat\theta_i$ then reset $C_i = C_i + 1$
					\If {$C_i = W$ and $\theta_{g,i}= 0$}
						\State Set $\theta_{u,i}= \hat\theta_i, \hat\theta_i = (\theta_{u,i} + \theta_{l,i})/2$, $C_i=0$. If available allocate resource $\hat{\theta}_i$
						\State If $\theta_{u,i} - \theta_{l,i}\le \gamma$ then set $\theta_{g,i}=1$ and $\hat\theta_i=\theta_{u,i}$
					\EndIf
				\EndIf
			\EndFor
			\While{free resources are available}
				\State Allocate $\hat\theta_i$ resources to a new randomly chosen arm $i$ from the arms having $\theta_{g,i}=1$
			\EndWhile
			\State $T_{\theta_d} = T_{\theta_d} + 1$
		\EndWhile
		\vspace{1mm}
		\Statex \textbackslash\textbackslash \textbf{ Regret Minimization Phase} \textbackslash\textbackslash
		\vspace{0.5mm}
		\State $\forall i \in [K]: S_i = 0, N_i = 1$
		\For{$t=T_{\theta_d}+1, T_{\theta_d} + 2, \ldots, T$}
			\State $\forall i \in [K]: \hat{\mu}_i(t) \leftarrow \max\left\{ \frac{S_i}{N_i} - \sqrt{\frac{1.5\log(t)}{N_i}}, 0 \right\}$
			\State $L_t \leftarrow$ Oracle$\big( KP(\hat\bmu(t), \hat\btheta,Q)\big)$ 
			\State $\forall i \in L_t:$ allocate no resource to arm $i$. $\forall j \in K\setminus L_t:$ allocate $\hat\theta_j$ resources to arm $j$
			\State $\forall i \in L_t:$ observe $X_{t,i}$. Update $S_i = S_i+X_{t,i}$ and $N_i = N_i+1$
		\EndFor
	\end{algorithmic}
\end{algorithm}

\subsection{Computation complexity of 0-1 Knapsack when items have fractional weight and value}
Even though $KP(\bmu,\btheta, Q)$ is NP-Hard problem; it can be solved by a pseudo-polynomial time algorithm\footnote{The running time of pseudo-polynomial time algorithm is a polynomial in the numeric value of the input whereas the running time of polynomial-time algorithms is polynomial of the length of the input.} using dynamic programming with the time complexity of O$(KQ)$. But such an algorithm for $KP(\bmu,\btheta, Q)$, works when the value and weight of items are integers. In case of $\mu_i$ and $\theta_i$ are fractions, they need to convert these to integers with the desired accuracy by multiplying by large value $S$. The time complexity of solving $KP(S\bmu, S\btheta, SQ)$ is O$(KSQ)$ as a new capacity of Knapsack is $SQ$. Therefore, the time complexity of solving $KP(S\bmu, S\btheta, SQ)$ in each of the $T$ rounds is O$(TKSQ)$. 

Note that the empirical mean losses do not change drastically in consecutive rounds in practice (except initial rounds). As solving $0$-$1$ knapsack is computationally expensive, we can solve it after $N$ rounds. We use $S = 10^4$ and $N = 20$ in our experiments involving different thresholds.
	\fi
	
\end{document}